\ifx\SubmitVer\undefined%
\newcommand{\InSubmitVer}[1]{}%
\newcommand{\InFullVer}[1]{#1}%
\else
\newcommand{\InSubmitVer}[1]{#1}%
\newcommand{\InFullVer}[1]{}%
\fi

\InSubmitVer{%
   \documentclass[11pt]{article}%
}
\InFullVer{%
   \documentclass[12pt]{article}%
}

\usepackage[cm]{fullpage}%
\usepackage{graphicx}%
\usepackage{amsmath}%
\usepackage{color}%
\usepackage{xspace}%
\usepackage{paralist}
\usepackage{mleftright}%
\usepackage{amssymb}%
\usepackage{xcolor}%
\usepackage{picins}%
\usepackage{caption}%

\usepackage{stmaryrd}%

\usepackage{abstract}
\usepackage{tikz}  % Used for the extended footnotes if needed
\usepackage{pifont}%
\usepackage[symbol*]{footmisc}

\newcommand*\circled[1]{\footnotesize\tikz[baseline=(char.base)]{
    \node[shape=circle,draw,inner sep=0.2pt] (char) {#1};}}

\DefineFNsymbols*{stars}[text]{%
   {\ding{172}} %
   {\ding{173}} %
   {\ding{174}} %
   {\ding{175}} %
   {\ding{176}} %
   {\ding{177}} %
   {\ding{178}} %
   {\ding{179}} %
   {\ding{180}} %
   {{\protect\circled{{\tiny 10}}}}%
   {{\protect\circled{{\tiny 11}}}}%
   {{\protect\circled{{\tiny 12}}}}%
   {{\protect\circled{{\tiny 13}}}}%
   {{\protect\circled{{\tiny 14}}}}%
   {{\protect\circled{{\tiny 15}}}}%
   {{\protect\circled{{\tiny 16}}}}%
   {{\protect\circled{{\tiny 17}}}}%
   {{\protect\circled{{\tiny 18}}}}%
   {{\protect\circled{{\tiny 19}}}}%
   {{\protect\circled{{\tiny 20}}}}%
   {{\protect\circled{{\tiny 21}}}}%
   {{\protect\circled{{\tiny 22}}}}%
}%
\usepackage[amsmath,thmmarks]{ntheorem}%
\theoremseparator{.}%

\usepackage{titlesec}%
\titlelabel{\thetitle. }%

\usepackage{hyperref}%
\hypersetup{%
   breaklinks,%
   colorlinks=true,%
   linkcolor=[rgb]{0.45,0.0,0.0},%
   citecolor=[rgb]{0,0,0.45}
}

\numberwithin{figure}{section}%
\numberwithin{table}{section}%
\numberwithin{equation}{section}%

\renewcommand{\Re}{\mathbb{R}}%
\theoremstyle{plain}%
\newtheorem{theorem}{Theorem}[section]

\newtheorem{lemma}[theorem]{Lemma}%

\theoremstyle{plain}%
\theoremheaderfont{\sf}%
\theorembodyfont{\upshape}%

\newtheorem*{remark:unnumbered}[theorem]{Remark}%

\newcommand{\HLinkShort}[2]{\hyperref[#2]{#1\ref*{#2}}}
\newcommand{\HLink}[2]{\hyperref[#2]{#1~\ref*{#2}}}
\newcommand{\HLinkPage}[2]{\hyperref[#2]{#1~\ref*{#2}%
      $_\text{p\pageref{#2}}$}}
\newcommand{\HLinkPageOnly}[1]{\hyperref[#1]{Page~\refpage*{#1}%
      $_\text{p\pageref{#1}}$}}

\newcommand{\HLinkSuffix}[3]{\hyperref[#2]{#1\ref*{#2}{#3}}}
\newcommand{\HLinkPageSuffix}[3]{\hyperref[#2]{#1\ref*{#2}%
      #3$_\text{p\pageref{#2}}$}}

\newcommand{\lemlab}[1]{\label{lemma:#1}}
\newcommand{\lemref}[1]{\HLink{Lemma}{lemma:#1}}%

\providecommand{\eqlab}[1]{}%
\renewcommand{\eqlab}[1]{\label{equation:#1}}
\newcommand{\Eqref}[1]{\HLinkSuffix{Eq.~(}{equation:#1}{)}}

\newcommand{\myqedsymbol}{\rule{2mm}{2mm}}
\theoremheaderfont{\em}%
\theorembodyfont{\upshape}%
\theoremstyle{nonumberplain} \theoremseparator{}
\theoremsymbol{\myqedsymbol} \newtheorem{proof}{Proof:}
   
\newcommand{\pbrc}[1]{\mleft[ {#1}  \mright]}

\DefineNamedColor{named}{RedViolet} {cmyk}{0.07,0.90,0,0.34}

\definecolor{blue25}{rgb}{0, 0, 11}%
\newcommand{\emphic}[2]{%
   \textcolor{blue25}{%
      \textbf{\emph{#1}}}%
   \index{#2}}

\newcommand{\emphi}[1]{\emphic{#1}{#1}}

\newcommand{\diamX}[1]{\mathrm{diam}\pth{#1}}%

\renewcommand{\th}{th\xspace}

\newcommand{\atgen}{\symbol{'100}}
\newcommand{\SarielThanks}[1]{\thanks{Department of Computer Science;
      University of Illinois; 201 N. Goodwin Avenue; Urbana, IL,
      61801, USA; {\tt sariel\atgen{}illinois.edu}; {\tt
         \url{http://sarielhp.org/}.} #1}}

\newcommand{\ceil}[1]{\left\lceil {#1} \right\rceil}

\newcommand{\etal}{\textit{et~al.}\xspace}

\newcommand{\pth}[1]{\mleft({#1}\mright)}%
\newcommand{\CHChar}{{\Math{\mathcal{C}}}}
\newcommand{\CHX}[1]{\CHChar_{#1}}

\newcommand{\distSet}[2]{\Math{d}\pth{#1, #2}}
\newcommand{\norm}[1]{\left\| {#1}  \right\|}
\newcommand{\distY}[2]{\norm{#1 - #2}}

\newcommand{\pnt}{\Math{{p}}}%

\newcommand{\PntSet}{\Math{P}}%

\newcommand{\ds}{\displaystyle}%
\newcommand{\bpnt}{\mathsf{b}}%
\newcommand{\wpnt}{\mathsf{w}}%

\newcommand{\slab}{S}%
\newcommand{\slabY}[2]{\mathrm{slab}\pth{#1, #2}}

\newcommand{\Math}[1]{{{#1}}}%

\newcommand{\BPntSet}{\mathsf{B}}%
\newcommand{\WPntSet}{\mathsf{W}}%

\newcommand{\diam}{\Math{\Delta}}%
\newcommand{\margin}{\Math{\gamma}}%

\newcommand{\eps}{\varepsilon}

\newcommand{\remove}[1]{}%

\IfFileExists{sariel_computer.sty}{\def\sarielComp{1}}{}
\ifx\sarielComp\undefined%
\newcommand{\SarielComp}[1]{}
\newcommand{\NotSarielComp}[1]{#1}%
 \else
\newcommand{\SarielComp}[1]{#1}%
\newcommand{\NotSarielComp}[1]{}%
\fi

\SarielComp{%
   \ifx\colorMath\undefined%
   \else
       \renewcommand{\Math}[1]{{\textcolor{red}{#1}}}
   \fi
}

\begin{document}

\title{A Simple Algorithm for Maximum Margin Classification,
   Revisited}

\author{%
   Sariel Har-Peled\SarielThanks{Work on this paper was partially
      supported by a NSF AF awards %
      CCF-1421231, and % Started June 2014
      CCF-1217462.  % Started June 2012
   }%
}

\date{\today}%

\maketitle%

% \setfnsymbol{stars}%
\setfnsymbol{stars}%

\begin{abstract}
    In this note, we revisit the algorithm of Har-Peled
    \etal~\cite{hrz-mmcan-07} for computing a linear maximum margin
    classifier. Our presentation is self contained, and the algorithm
    itself is slightly simpler than the original algorithm. The
    algorithm itself is a simple Perceptron like iterative algorithm.
    For more details and background, the reader is referred to the
    original paper.
\end{abstract}

\section{Active learning, sparsity and large margin}

Let $\PntSet$ be a point set of $n$ points in $\Re^d$. Every point has
a label/color (say \emph{black} or \emph{white}), but we do not know
the labels.  In particular, let $\BPntSet$ and $\WPntSet$ be the set
of black and white points in $\PntSet$. Furthermore, let
$\diam = \diamX{\PntSet}$, and assume that there exist two parallel
hyperplanes $h, h'$ in distance $\margin$ from each other, such that
the slab between $h$ and $h'$ does not contain an point of $\PntSet$,
and the points of $\BPntSet$ are on one side of this slab, and the
points of $\WPntSet$ are on the other side. The quantity $\margin$ is
the \emphi{margin} of $\PntSet$.

A somewhat more convenient way to handle such slabs, is to consider
two points $\bpnt$ and $\wpnt$ in $\Re^d$. Let $\slabY{\bpnt}{\wpnt}$
be the region of points in $\Re^d$, such that their projection onto
the line spanned by $\bpnt$ and $\wpnt$ is contained in the open
segment $\bpnt\wpnt$. We use $(1-\eps)\slabY{\bpnt}{\wpnt}$ to denote
the slab formed from $\slabY{\bpnt}{\wpnt}$ by shrinking it by a
factor of $(1-\eps)$ around its middle hyperplane. Formally, it is
defined as $(1-\eps)\slabY{\bpnt}{\wpnt} = \slabY{\bpnt'}{\wpnt'}$,
where $\bpnt' = (1-\eps/2)\bpnt +(\eps/2) \wpnt$ and
$\wpnt' = (\eps/2)\bpnt +(1-\eps/2) \wpnt$.

In the following, we assume have an access to a \emphi{labeling
   oracle} that can return the label of a specific query
point. Similarly, we assume access to a \emphi{counterexample oracle},
such that given a slab that does not contain any points of $\PntSet$
in its interior, and supposedly separates the points of $\PntSet$ into
$\BPntSet$ and $\WPntSet$, it returns a point that is mislabeled by
this classifier (i.e., slab) if such a point exists.

Conceptually, asking queries from the oracles is quite expensive, and
the algorithm tries to minimize the number of such queries.

\paragraph{The algorithm.}
Assume there are two points $\bpnt_1 \in \BPntSet$ and
$\wpnt_1 \in \WPntSet$.  For $i > 0$, in the $i$\th iteration, the
algorithm considers the slab
\begin{math}
    \slab_i = (1-\eps)\slabY{\bpnt_{i}}{\wpnt_{i}}.
\end{math}
There are two possibilities:
\begin{compactenum}[\quad(A)]
    \item If the slab $\slab_i$ contains no points of $\PntSet$, then
    the algorithm uses the counterexample oracle to check if it is
    done -- that is, all the points are classified
    correctly. Otherwise, a badly classified point $\pnt_i$ was
    returned.
    
    \item The $\slab_i$ contains some points of $\PntSet$, and let
    $\pnt_i$ be the closest point to the middle hyperplane of the slab
    $\slab_i$. The algorithm uses the labeling oracle to get the label
    of $\pnt_i$.
\end{compactenum}
Assume that the label of $\pnt_i$ is white. Then, the algorithm set
$\wpnt_{i+1}$ be the projection of $\bpnt_{i}$ to $\wpnt_{i} \pnt_i$,
and $\bpnt_{i+1} = \bpnt_{i}$ (the case that $\pnt_i$ is black is
handled in a symmetric fashion).

\begin{lemma}[\cite{hrz-mmcan-07}]
    \lemlab{active}%
    Let $\PntSet$ be a set of points in $\Re^d$, with diameter
    $\diam$. Assume there is an unknown partition of $\PntSet$ into
    two (unknown) point sets $\BPntSet$ and $\WPntSet$, of white and
    black points, respectively, and this partition has margin
    $\margin$. Furthermore, we are given an access to a labeling and
    counterexample oracles. Finally, there are two given points
    $\bpnt_1 \in \BPntSet$ and $\wpnt_1 \in \WPntSet$.

    Then, for any $\eps > 0$, one can compute using an iterative
    algorithm, in $I = O\pth{\Bigl. \pth{\diam /\margin}^2 /\eps^2 }$
    iterations and in $O(I dn )$ time, a slab of width
    $\geq (1-\eps)\margin$ that separates $\BPntSet$ from
    $\WPntSet$. This algorithm performs $I$ calls to the
    labeling/counterexample oracles.
\end{lemma}

\begin{proof}   
    Our purpose is to analyze the number of iterations of this
    algorithm till it terminates.  So, let
    $\ell_i = \distY{\bpnt_i}{\wpnt_i}$. Clearly,
    \begin{math}
        \diam \geq \ell_0 \geq \ell_1 \geq \cdots \geq \margin,
    \end{math}
    the last step follows as $\bpnt_i \in \CHX{\BPntSet}$ and
    $\wpnt_i \in \CHX{\WPntSet}$, and the distance
    $\distSet{\CHX{\BPntSet}}{\CHX{\WPntSet}} \geq \margin$, where
    $\distSet{X}{Y} = \min_{x\in X} \min_ {y \in Y} \distY{x}{y}$.

    \parpic[r]{\includegraphics{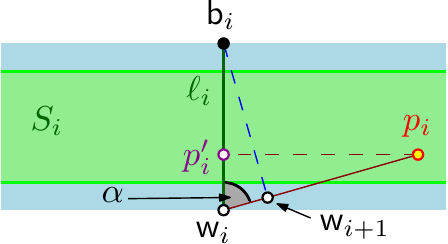}}%
    Let $\pnt_i'$ be the projection of $\pnt_i$ to the line spanned by
    $\wpnt_i \bpnt_i$. Observe that if $\pnt_i \in \slab_i$ then
    \begin{math}
        \distY{\pnt_i'}{\wpnt_i} \geq \eps \ell_i / 2.
    \end{math}
    Formally, the points $\wpnt_i$ breaks the line spanned by
    $\wpnt_i$ and $\bpnt_i$ into two parts, and $\bpnt_i$ and
    $\pnt_i'$ are on the same side, and $\pnt_i'$ is distance at least
    $\ell_i/2$ away from $\wpnt_i$ along this ray. Observe that if
    case (B) above happened, then $\pnt_i$ is not inside $\slab_i$,
    and this distance is significantly larger.

    Setting $\alpha = \angle \pnt_i \wpnt_i \bpnt_i$, we have
    \begin{math}%
        \ds %
        \cos \alpha%
        =%
        \frac{\distY{\pnt_i'}{\wpnt_i}}{\distY{\wpnt_i}{\pnt_i}}%
        \geq%
        \frac{\eps \ell_i /2}{\diam}.
    \end{math}
    As such, we have
    \begin{align}
        \ds%
        \ell_{i+1}%
        =%
        \ell_i \sin \alpha%
        \leq %
        \ell_i \sqrt{1 - \pth{\frac{\eps \ell_i}{2 \diam}}^2 }%
        \leq %
        \pth{1 - \pth{\frac{\eps \ell_i}{4\diam}}^{\!2} } \ell_i.%
    \end{align}

    We have that $\ell_{i+k} \leq \ell_i/2$, for
    $k = \ceil{64 {\Bigl. \diam^2 /(\eps\ell_i)^2 }}$. Indeed, if
    $\ell_{i+k} > \ell_i/2$, then
    \begin{align}
        \ell_{i+k}%
        &\leq%
        \ell_i%
        \prod_{j=0}^{k-1} \pth{%
           1- \pth{ \frac{\eps \ell_{i+j}}{4\diam} }^{\!2} }%
        \leq%
        \ell_i%
        \prod_{j=0}^{k-1} \pth{%
           1- \pth{ \frac{\eps \ell_{i+k}}{4\diam} }^{\!2} }%
        \leq%
        \ell_i%
        \exp \pth{%
           - k \pth{ \frac{\eps \ell_{i+k}}{4\diam} }^{\!2} }%
        \\%
        &\leq%
        \ell_i%
        \exp \pth{%
           - k \pth{ \frac{\eps \ell_{i}}{8\diam} }^{\!2} }%
        \leq%
        \ell_i%
        \exp \pth{ - k \pth{ \frac{\eps \ell_{i}}{8\diam} }^{\!2} }%
        \leq%
        \frac{\ell_i}{e},
        \eqlab{you:do:not:say}%
    \end{align}
    which is a contradiction.

    In particular, the $j$\th epoch of the algorithm are the
    iterations where $\ell_i \in \pbrc{\diam/2^{j-1}, \diam/ 2^{j}}$.
    Namely, during an epoch the width of the current slab shrinks by a
    factor of two. By \Eqref{you:do:not:say}, the $j$\th epoch lasts
    \begin{math}
        n_j = O\pth{ \pth{2^j/\eps}^{2}}
    \end{math}
    iterations. As such, the total number of iterations $\sum_j n_j$
    is dominated by the last epoch, that starts (roughly) when
    $\ell_i \leq 2\margin$, and end when it hits $\margin$. This last
    epoch takes $O\pth{\Bigl. \diam^2 /(\eps\margin)^2 }$ iterations,
    which also bounds the total number of iterations.
\end{proof}

\begin{remark:unnumbered}
    (A) if the data is already labeled, then the algorithm of
    \lemref{active} can be implemented directly resulting in the same
    running time as stated. This algorithm approximates the maximum
    margin classifier to the data.  Specifically, the above algorithm
    $(1+\eps)$-approximates the distance
    $\distSet{\BPntSet}{\WPntSet}$, and it can be interpreted as an
    approximation algorithm for the associated quadratic program.

    (B) One can implement the counterexample oracle, by sampling
    enough labels, and using the labeling oracle. This is introduces a
    certain level of error. See \cite{hrz-mmcan-07} for details.
\end{remark:unnumbered}

%*flatex input: [./active.bbl]

% flatex input end: [./active.bbl]
%FLATEX-REM:\bibliography{active}%

\end{document}